\documentclass[runningheads]{llncs}

\usepackage[utf8]{inputenc} 
\usepackage[T1]{fontenc}    
\usepackage{hyperref}       
\usepackage{url}            
\usepackage{booktabs}       
\usepackage{amsfonts}       
\usepackage{nicefrac}       
\usepackage{microtype}      
\usepackage{xcolor}         

\usepackage[american]{babel}
\usepackage{amsmath}
\usepackage{amssymb}

\usepackage[font=small,labelfont=bf]{caption}
\usepackage{subcaption}

\usepackage{color, colortbl}
\usepackage{graphicx}

\usepackage{enumitem}
\usepackage{todonotes}

\newcommand{\squishlist}{
 \begin{list}{$\bullet$}
  { \setlength{\itemsep}{0pt}
     \setlength{\parsep}{1pt}
     \setlength{\topsep}{1pt}
     \setlength{\partopsep}{0pt}
     \setlength{\leftmargin}{1em}
     \setlength{\labelwidth}{1em}
     \setlength{\labelsep}{0.5em} } }
 \newcommand{\squishend}{\end{list}}

\usepackage{caption}
\usepackage[nameinlink]{cleveref}

\Crefname{lemma}{Lemma}{Lemmas}

\begin{document}

\title{Robust and Provable Guarantees for\\ Sparse Random Embeddings}
\author{Maciej Skorski\inst{1}\and Alessandro Temperoni\inst{1} \and Martin Theobald\inst{1}}

\titlerunning{Robust \& Provable Guarantees for Sparse Random Embeddings}
\authorrunning{M. Skorski \and A. Temperoni \and M. Theobald}

\institute{Department of Computer Science\\ University of Luxembourg\\4365 Esch-sur-Alzette, Luxembourg} 

\maketitle
\begin{abstract}
In this work, we improve upon the guarantees for sparse random embeddings, as they were recently provided and analyzed by Freksen at al.  (NIPS'18) and Jagadeesan (NIPS'19). Specifically, we show that (a) our bounds are {\em explicit} as opposed to the asymptotic guarantees provided previously, and (b) our bounds are guaranteed to be {\em sharper} by practically significant constants across a wide range of parameters, including the {\em dimensionality}, {\em sparsity} and {\em dispersion} of the data. Moreover, we empirically demonstrate that our bounds significantly outperform prior works on a wide range of real-world datasets, such as collections of images, text documents represented as bags-of-words, and text sequences vectorized by neural embeddings.
Behind our numerical improvements are techniques of broader interest, which improve upon key steps of previous analyses in terms of (c) tighter estimates for certain types of {\em quadratic chaos}, (d) establishing extreme properties of {\em sparse linear forms}, and (e) improvements on bounds for the estimation of {\em sums of independent random variables}.
\end{abstract}

\section{Introduction}
\label{sec:intro}

\subsection{Background: Random Embeddings}
The seminal result of Johnson and Lindenstrauss~\cite{johnson1984extensions} states that {\em random linear mappings} have nearly isometric properties, and hence are well-suited for embeddings: they \emph{nearly preserve distances} when projecting high-dimensional data into a \emph{lower-dimensional space}. Formally, 
for an error parameter $\epsilon>0$, an $m\times n$ matrix $A$ appropriately sampled (e.g., using appropriately scaled Gaussian entries), and any input vector $x\in\mathbb{R}^n$, it holds that
\begin{align}\label{eq:jl_statement}
1-\epsilon\leqslant \|A\,x\|_2 / \|x\|_2 \leqslant 1+\epsilon\quad \textrm{ with probability } 1-\delta
\end{align}
 if the embedding dimension is $m = \Theta\left( \frac{1}{\epsilon^2} \log\frac{1}{\delta}\right)$.

This bound on the dimension $m$ has been shown to be {\em asymptotically optimal}~\cite{jayram2013optimal,kane2011almost}, while the assumptions made on the Gaussian distribution of matrix $A$ can be further replaced by the Rademacher distribution (when the $\pm 1$ are randomly sampled)~\cite{achlioptas2001database}, or relaxed even further by only requiring the sub-Gaussian condition to hold for the construction of the projection matrix $A$~\cite{boucheron2013concentration}. 

The result is a \emph{dimension-distortion tradeoff}: one aims to minimize $m \ll n$, while keeping $\epsilon$ and $\delta$
possibly small. Smaller dimensions $m$ allow for efficient processing of large, high-dimensional datasets, while a small distortion guarantees that analytical tasks can be performed with a similar effect over the embedded data as it is the case for the original data (which can perhaps easiest be seen by the example of \emph{Cosine similarity} frequently used for clustering and data-mining and/or various machine-learning techniques~\cite{tan2016introduction}). 

Over the past years, variants of the aforementioned \emph{Johnson-Lindenstrauss Lemma} have found important applications to text mining and image processing~\cite{bingham2001random}, approximate nearest-neighbor search~\cite{ailon2006approximate,indyk1998approximate}, learning mixtures of Gaussians~\cite{dasgupta1999learning}, sketching and streaming algorithms~\cite{kerber2014approximation,kpotufe2020gaussian}, approximation algorithms for clustering high-dimensional data~\cite{biau2008performance,boutsidis2010random,makarychev2019performance}, speeding up computations in linear algebra~\cite{clarkson2017low,nelson2013osnap,sarlos2006improved}, analyzing graphs~\cite{frankl1988johnson,linial1995geometry},
and even to hypothesis testing~\cite{lopes2011more,shi2020sparse} and data privacy~\cite{blocki2012johnson,kenthapadi2013privacy}. From a theoretical perspective, the importance of understanding Hilbert spaces in functional analysis~\cite{johnson2010johnson} is also worth mentioning. Finally, we note that, while \Cref{eq:jl_statement} gives high-probability guarantees for a randomly sampled matrix $A$, it is in fact possible to construct a concrete matrix $A$ which (surely) satisfies this inequality in randomized polynomial time~\cite{dasgupta1999elementary} or by means of derandomization~\cite{kane2010derandomized}.

The particular focus of this paper is on \emph{linear sparse random embeddings}, where $A$ in \Cref{eq:jl_statement} has at most $s$ non-zero entries in each column, which allows for faster computation of the embedded vectors. This setup has been covered by a substantial line of recent research~\cite{achlioptas2001database,ailon2006approximate,cohen2016nearly,dasgupta2010sparse,kane2014sparser,li2006very,matouvsek2008variants}, which established that, for the optimal dimension $m$, one can set $s=\Theta(m \, \epsilon)$, thereby gaining a factor of $\epsilon$ in matrix sparsity\footnote{As noticed in~\cite{cohen2016nearly}, one may further reduce the sparsity $s$ by a factor of $B>1$, however at the cost of increasing the dimension $m$ by a factor of $2^{\Theta(B)}$ (i.e., exponentially).}. 
This idea can be further improved by exploiting \emph{structural properties} of the input data: as shown in recent works~\cite{dasgupta2010sparse,freksen2018fully,jagadeesan2019understanding,kane2014sparser,weinberger2009feature}, with $v \triangleq \|x\|_{\infty}/\|x\|_2$, one may set the sparsity to
\begin{align}
s=\Theta\left(\frac{v^2}{\epsilon}\max\left(\log\frac{1}{\delta},\frac{ \log^2\frac{1}{\delta}}{ \log\frac{1}{\epsilon}) }\right) \right)
\end{align}
while keeping the optimal choice of dimension $m = \Theta\left( \frac{1}{\epsilon^2} \log\frac{1}{\delta}\right)$.
This shows that better sparsity $s$ (possibly including the extreme case of $s=1$, which then essentially becomes equivalent to \emph{feature hashing} in machine learning~\cite{dahlgaard2017practical,weinberger2009feature}) is possible when the data-dependent parameter $v$ is small.
The parameter $v$ should thus be understood as the {\em dispersion} of the input vector $x$, i.e., $v$ is small when the components of $x$ are of comparable magnitude, and it is larger when there are dominating components. This result aligns with an intuitive understanding of sums of random components: they converge slower in presence of a ``large dispersion''.

\subsection{Motivation: Why Do Random Embeddings Work So Well In Practice?}
Our work is motivated by the general observation that random embeddings empirically often work much better than it is predicted by their theoretical bounds. The main goal of this work thus is to bridge this frequently observed gap between theory and practice and thereby develop both \emph{robust} and {\em provable guarantees} for sparse random embeddings. Remarkably, despite the huge progress in the provided analyses, no prior work so far has been able to match the theoretical guarantees with the empirically observed---very good---performance of sparse random embeddings~\cite{akusok2018comparison,venkatasubramanian2011johnson}.
The demand for provable guarantees does not only come from theory, but also from applied data science: the conservative estimates on what good parameters are also appear in various popular machine-learning libraries such as \texttt{Scikit-learn}~\cite{scikit-learn}.

Recent state-of-the-art analyses~\cite{jagadeesan2019understanding,freksen2018fully} 
are quite involved in terms of their dependencies on other results, and they provide only asymptotic bounds which tend to disguise dependencies on rather large constants~\cite{freksen2018fully}. In practice, they often yield trivial results which however limits their usability. Moreover, real-data evaluations from prior works are mostly of qualitative nature: they analyze trends in parameter tradeoffs~\cite{jagadeesan2019understanding} rather than provable guarantees. 
Regarding the dispersion measure $v=\|x\|_{\infty}/\|x\|_{2}$, which is the key ingredient of recent improvements, no study has evaluated its typical behavior on real-world data to our knowledge so far. It appears that, in practice, $v$ may be too large to justify the desired values of $s$. The typical value of the dispersion $v$ may also depend on the type of the data (text, images, etc.), which in turn makes the findings harder to generalize.

The literature offers no satisfactory treatment of this prevalent gap between provable and practically meaningful guarantees. Some authors~\cite{venkatasubramanian2011johnson,freksen2018fully} suggested that very good empirical performance may be an evidence for small constants, but it may well be the case that sparse random embeddings work better than predicted by the underlying theory due to other data properties, not present in any of the analyses. Indeed, while one can expect the low data dispersion to help increasing sparsity, the proposed measure $v$ is very crude and does not capture this aspect well in a quantitative sense. We also note that ``optimality'' of bounds from prior works~\cite{freksen2018fully,jagadeesan2019understanding} is to be understood in a somewhat narrow sense: asymptotically and from a structure-agnostic perspective, i.e., when we do not have any more fine-grained information about the input data $x$ that would go beyond $v$. 

\subsection{Contributions}
We summarize the novel contributions of our work as follows.

\subsubsection{Explicit Analysis} We re-analyze sparse random embeddings, following the setup of recent state-of-the-art works~\cite{freksen2018fully,jagadeesan2019understanding} which provide guarantees depending on the data dispersion $v=\|x\|_{\infty}/\|x\|_2$. Our novel bound is a combinatorial expression that is \emph{computationally fast to evaluate}. More precisely, our expression on the error term $\epsilon$ can be evaluated in nearly constant time of $O(\log^{4}(1/\delta)\log(m/\delta))$ operations. In our implementation (we use \texttt{Google Colab}) such a call takes about one millisecond on average.

\subsubsection{Robust \& Provable Guarantees}
We demonstrate that our bounds are very robust and accurate over a large variety of practical use-cases as well as over a wide range of dispersion values $v$ and error bounds $\epsilon$. In particular, they consistently outperform prior works~\cite{freksen2018fully,jagadeesan2019understanding} even if these are provided with ``optimistic'' constants. Moreover, we give an exhaustive evaluation on both a synthetic benchmark and no less than 10 real-world datasets concerning text in different representations, images of various sizes, and sparse matrices which arise in typical scientific computations. We see improvements by a factor of more than one-order-of-magnitude in the projected dimension $m$ and sparsity $s$ of $A$, and even more in the confidence $1-\delta$.

\subsubsection{Techniques of Further Interest}
Behind our numerical improvements are also novel theoretical results of general interest, which substantially improve upon the key steps in the previous analyses. We summarize these results as follows.

\paragraph{Improved estimation of quadratic chaos.~}
Virtually all analyses of random embeddings need to estimate quadratic forms in symmetric random variables, which arise due to considering the Euclidean distance of the projected vector. To solve this problem, we give a novel bound for the quadratic form in terms of its linear analogue, with a very good numerical constant. This improves upon direct estimation from prior work~\cite{jagadeesan2019understanding}, as well as (in this context) general-purpose tools such as variants of the Hanson-Wright Lemma~\cite{hanson1971bound} and decoupling inequalities~\cite{de1993bounds,vershynin2011simple}.

\paragraph{Extremal properties of sparse linear chaos.~}
The reduction from a quadratic form leaves us with the task of understanding stochastic properties of certain random sums,
namely the inner product of the given weight vector  (our input $x$) and a random vector with entries $-1$, $1$ or $0$ (i.e., one row of the matrix $A$). This problem is related to, but more general than the well-known Khintchine Inequality~\cite{khintchine1923dyadische,hitczenko1993domination}. In our context (i.e., providing bounds based on data dispersion), we explicitly find the worst-performing set of weights, as opposed to prior work where only overestimates were obtained~\cite{jagadeesan2019understanding}. To this end, we use the geometric technique of \emph{majorization}~\cite{schur1923uber}, which gives very precise insights into the stochastic behavior of such sparse transformations with respect to the weights.

\paragraph{Estimation of sums of i.i.d. random variables.~}
To derive accurate bounds, we rely on a precise estimation of sums of independent random variables which goes beyond what is offered by classical Chernoff-Heoffding bounds. Remarkably, we are able to numerically improve the state-of-the-art bounds due to Latala~\cite{latala1997estimation}, which further adds to the success of our approach on real-world data.

\subsection{Related Work}

\subsubsection{Theory of Sparse Random Embeddings}
Our work improves directly upon \cite{freksen2018fully} (case $s=1$) and \cite{jagadeesan2019understanding} (general $s$). These works determine the relation between sparsity and the data dispersion, building on a long line of earlier works based on variants of the Johnson-Lindenstrauss Lemma~\cite{achlioptas2001database,ailon2006approximate,cohen2016nearly,dasgupta2010sparse,kane2014sparser,li2006very,matouvsek2008variants}. The provided bounds, albeit proven to be asymptotically optimal, suffer from the \emph{lack of explicit constants} which cannot be easily extracted due to several imprecise estimates. As noted in \cite{freksen2018fully}, while the constants seem astronomically big, the empirical performance gives hope for tightening the bounds.

\subsubsection{Empirical Evaluation of Random Embeddings}
The good empirical performance of random linear embeddings, including sparse variants, has been confirmed many times (see~\cite{akusok2018comparison,bingham2001random,venkatasubramanian2011johnson}). These works point out the gap between provable and observed performance, which we are addressing also in this work. All these works agree that random embeddings perform much better in practice than predicted by the underlying theory.

\subsubsection{Estimation of Quadratic Chaos}
Technically speaking, the analysis of errors in random projections can be reduced to the more general problem of estimating 
quadratic forms of random variables, also called \emph{quadratic chaos}. The literature offers a variety 
of tools, from variants of the well-known Hanson-Wright Lemma~\cite{hanson1971bound,zhou2019sparse} to more specialized bounds~\cite{boucheron2005moment,latala1999tail}. However, these would produce worse constants than our direct approach.

\subsubsection{Embeddings of Big Collections/Subspaces}
Orthogonal to obtaining bounds for a single vector $x$ is the question of how to extend such bounds to hold simultaneously 
for all $x$ from a finite collection or an entire subspace of input data. This can be done by a black-box reduction
using $\epsilon$-net arguments  (see the works on subspace embeddings~\cite{cohen2016nearly,sarlos2006improved}) and is solved by a reduction to the single vector case by means of $\epsilon$-net arguments. Such bounds can be also obtained in our case.

\subsubsection{Monte-Carlo Simulations}
A direct way to accurately estimate the performance of the embeddings would be to evaluate their probability of distortion on a representative sample of the data. This is however both computationally costly and methodologically complex. To explain the second issue, we note that regularity assumptions have to be carefully verified to claim concrete statistical evidence (likely, robust estimation~\cite{lugosi2019mean} would have to be used to avoid possible problems with heavier tails).

\subsubsection{Low-Distortion Embeddings}
Finally, there is also lots of research on general (not necessarily linear) embeddings that almost preserve various (not necessarily Euclidean) distances. See for example ~\cite{an2015can,mcqueen2016nearly} and also~\cite{indyk2001algorithmic} for a survey of algorithmic applications.

\section{Robust Guarantees for Sparse Random Projections}

We next briefly discuss a number of preliminary concepts, mainly to fix the notation before we move on to present the main results of our work.

\subsection{Preliminaries \& Notation}
The $d$-th {\em norm} of a vector $x$ and a random variable $X$, respectively, are defined as $\|x\|_d = (\sum_i |x_i|^{d})^{\frac{1}{d}}$ and $\|X\|_d = \left(\mathbb{E}\left[|X|^d\right]\right)^{\frac{1}{d}}$; we also define $\|x\|_{\infty} = \max_i |x_i|$ as usual.  $\mathrm{Bern}(p)$ denotes the {\em Bernoulli distribution} with success probability $p$,
while $\mathrm{Binom}(n,p)$ denotes the {\em binomial distribution} with $n$ trials and success probability $p$. The {\em Rademacher distribution} takes values $1$ and $-1$ with equal probabilities. Moreover, a random variable $X$ is called {\em symmetric} when it has the same distribution as $-X$.
For two vectors $x, y \in \mathbb{R}^n$, we say that $x$ {\em majorizes} $y$, denoted by $x\succ y$, when $\sum_{i=1}^{k}x^{\downarrow}_i\geqslant \sum_{i=1}^{k} y^{\downarrow}_i$, for $k=1\ldots n$. Finally, {\em Schur-concave} functions $f$ are those that satisfy $f(x) \leqslant f(y)$ whenever $x \succ y$.

\subsection{Construction of the Embeddings}\label{sec:construction}

Let $A$ be an $m \times n$ matrix which is sampled as follows:

\smallskip
\begin{centering}
\fbox{\parbox{0.98\linewidth}{
\begin{itemize}
\item[(1)] Fix a positive integer $s \leqslant m$, the {\em column sparsity} of $A$.
\item[(2)] For each column, select $s$ row positions at random (without replacement), place $\pm 1$ uniform-randomly at these positions and 0 at the remaining positions.
\item[(3)] Finally, scale all entries of $A$ by $\frac{1}{\sqrt{s}}$.
\end{itemize}
}}
\end{centering}

\begin{remark}[Alternative Constructions]
The above construction of $A$ is as in~\cite{freksen2018fully,jagadeesan2019understanding}, but our analysis works also when we select $s$ non-zero entries in the rows (rather than in the columns) of $A$, or when sampling is done with replacement. 
\end{remark}

\noindent To analyze the error obtained from the respective projection of $x$ by $A$, we define as in \cite{jagadeesan2019understanding}
\begin{align}\label{eq:contrib}
E(x) &\triangleq \|A x\|_2^2 - \|x\|_2^2 =\sum_{r=1}^{m} \sum_{1\leqslant i\not=j\leqslant n} A_{r,i}\,A_{r,j}\,x_i\,x_j
\end{align}
which is then analyzed by looking into individual ``row'' contributions, namely $E(x) = \frac{1}{s}\sum_{r=1}^{m}E_r(x)$ with
\begin{align}\label{eq:row_contrib}
E_r(x) &\triangleq ~~s~ \sum_{1\leqslant i\not=j\leqslant n} A_{r,i}\,A_{r,j}\,x_i\,x_j\,.
\end{align}
The goal is to identify conditions such that  $\Pr_{A}[|E(x)| > \epsilon \|x\|_2^2]\leqslant \delta$, as this implies \Cref{eq:jl_statement}.
By scaling, we can assume $\|x\|_2=1$. Throughout the paper, we denote $p=\frac{s}{m}$. 

\subsection{Key Techniques for the Analysis}

For the following steps, we leverage two techniques which were not used in prior work, namely (a) {\em careful use of symmetry properties} and (b) {\em majorization}.

\subsubsection{Quadratic Chaos Estimation}

Studying the error $E(x)$, due to pairwise terms, requires the estimation of quadratic forms $\sum_{i\not=j}Z_i\,Z_j$, with
$Z_i = A_{r,i}\,x_i$. To this end, we develop a useful general inequality, which reduces the problem to (simpler) linear forms.\footnote{Detailed proofs are provided as part of the Appendix of this paper.}

\begin{lemma}\label{lemma:chaos}
For symmetric and independent random variables $Z_i$ and any positive even $d$, we have:
\begin{align}
\|\sum_{i\not=j} Z_i\,Z_j\|_d \leqslant 4 \, \|\sum_i Z_i\|_{d}^2
\end{align}
\end{lemma}

\begin{remark}\label{rem:chaos}
Our proof establishes more, namely that for a positive integer $d$ (odd or even), we have $\|\sum_{i\not=j} Z_i\,Z_j\|_d \leqslant 4 \, \|\sum_{i\not=j}Z_i Z'_j\|_d$ where $Z'_i$ are independent copies of $Z_i$.
\end{remark}

Specifically, the proof (see~\ref{proof:lemma:chaos}) uses the well-known decoupling technique for quadratic forms~\cite{de1993bounds,vershynin2011simple}. We note that this bound is sharper than its analogue from~\cite{jagadeesan2019understanding}. The constant $C=4$ in \Cref{lemma:chaos} can be further improved. For example, it is easily seen that for $d=2$ one may choose $C=\sqrt{2}$. For a general $d$, the use of hypercontractive inequalities may give furthers refinements.

\subsubsection{Extremal Properties of Linear Chaos}

We now move on to deriving bounds for linear forms of symmetric random variables, which (as discussed before) bound quadratic forms. The following lemma gives a geometric insight into their behavior with respect to the input weights, which (in our case) are given by the input vector $x$.

\begin{lemma}\label{lemma:chaos_schur}
For $x\in\mathbb{R}^n$, define $S(x) = \sum_i x_i\,Y_i$
where $Y_i \sim^{\mathrm{i.i.d.}} Y$ with $Y\in \{-1,0,1\}$ taking values $\pm 1$ each with probability $p/2$ and $0$ with probability $1-p$. Then, for every pair of vectors $x, x'$ such that $(x_i^2)_i\succ ({x_i'}^2)_i$ and positive even integer $d$, the following inequality holds:
\begin{align}
\|S(x)\|_d \leqslant \|S(x')\|_d
\end{align}
\end{lemma}

We prove \Cref{lemma:chaos_schur} (see~\ref{proof:lemma:chaos_schur}) by using results from majorization theory~\cite{ostrowski1952quelques,schur1923uber}. 
We note that this extends~\cite{eaton1970note}, where only the case of $p=1$ has been studied.
The lemma yields the following corollary.

\begin{corollary}\label{cor:extreme_point}
Let $Y_i$ be as in~\Cref{lemma:chaos_schur}. For $v \in (0,1)$, consider all vectors $x \in \mathbb{R}^n$ such that $\|x\|_2=1$ and $\|x\|_{\infty}=v$. Then, $\|\sum_i x_i\, Y_i\|_d$ for an even $d>0$ is maximized at $x=x^{*}$ where:
\begin{align}
x^{*}_i = \begin{cases}
v & i=1 \\
\sqrt{\frac{1-v^2}{n-1}} & i=2\ldots n
\end{cases}
\end{align}
\end{corollary}
The result shows that the maximizing weights $x^{*}_i$ are as dispersed as possible (within the constraints). 


\subsubsection{Estimation of I.I.D. Sums}

The techniques outlined above allow us to bound the row-wise error contributions $E_r(x)$. 
In order to assemble them into a bound on the overall error $E(x)$, we prove the following lemma.
\begin{lemma}\label{lemma:iid_estimation}
Let $Z_1,\ldots,Z_m\sim^{i.i.d.} Z$, where $Z$ is symmetric, and let $d$ be positive and even. Then:
\begin{align}
\|\sum_{i=1}^{m}Z_i\|_d \leqslant \min\left\{t>0: \mathbb{E}(1+Z/t)^d \leqslant \mathrm{e}^{\frac{d}{2m}}\right\}
\end{align}
\end{lemma}

This improves the constant provided in the seminal result of Lata{\l}a~\cite{latala1997estimation} by a factor of $\mathrm{e}^{1/2}$.

\subsection{Bounds Based on Error Moments}

We first bound the row-wise error contributions $E_r(x)$, defined in \Cref{eq:row_contrib}, as follows.
\begin{lemma}\label{lemma:row_estimate}
Suppose that $\|x\|_2=1$ and $\|x\|_{\infty}=v$, then we have $\|E_r(x)\|_d\leqslant T_{n,p,d}(v)$ for any positive and even $d$, where we define
\small{
\begin{align}\label{eq:row_estimate}
T_{n,p,d}(v)\triangleq 4 \, \bigg(\sum_{k=0}^{\frac{d}{2}} \, \binom{d}{2k} \, p^{\mathbb{I}(k>0)} \, v^{2k}(1-v^2)^{\frac{d-2k}{2}} 
~\cdot~ \mathbb{E}(B'-B'')^{d-2k}\bigg)^{\frac{2}{d}}
\end{align}
}
and $B',B'' \sim^{i.i.d.} \frac{1}{\sqrt{n-1}} \cdot \mathrm{Binom}(n-1,\frac{1-\sqrt{1-2p}}{2})$.
\end{lemma}
To show this result, we combine \Cref{lemma:chaos} and \Cref{lemma:chaos_schur}.
When explicitly evaluating $\|\sum_i x^{*}_i Y_i\|_d$, we thereby arrive at the expression given by \Cref{eq:row_estimate}.

\medskip
Now we are in position to show the following theorem, which constitutes the main result of our work.

\begin{theorem}[Error Moments]\label{thm:main}
If $\|x\|_2=1$ and $\|x\|_{\infty}=v$, then for any positive even $d$, we have that
\begin{align*}
\| E(x) \|_d \leqslant  s^{-1}\cdot  Q_{n,p,d}(v),
\end{align*}
where $Q = Q_{n,p,d}(v)$ solves the equation
\begin{align}\label{eq:aggregated_func}
\sum_{k=0}^{\frac{d}{2}} \, \binom{d}{2k} \, (T_{n,p,2k}(v)/Q)^{2k}
= \mathrm{e}^{\frac{d}{2m}}
\end{align} 
and $T_{n,p,2k}$ is as in \Cref{lemma:row_estimate} (with $d$ replaced by $2k$).
\end{theorem}

The detailed proof (see~\ref{proof:thm:main}) starts with $E(x) =\frac{1}{s} \sum_{r=1}^{m}E_r(x)$, applies \Cref{lemma:iid_estimation} with $Z_r = E_r(x)$, and finally uses \Cref{lemma:row_estimate} (similarly to \cite{jagadeesan2019understanding}). The subtle points of the proof are summarized below.
\squishlist
\item\textit{Correlation of $E_r(x)$ for different $r$:} fortunately (due to sampling without replacement), this is a \emph{negative dependency}~\cite{dubhashi1996balls}. Thus, the same moment bounds as for independent random variables can be applied also here~\cite{shao2000comparison}.
\item\textit{Non-symmetric distribution of $E_r(x)$:} we compare the moments of $E_r(x)$ with the moments of a random variable which is symmetric; this allows for applying moment bounds for the sums of symmetric random variables. We remark that this argument also fills a gap in~\cite{jagadeesan2019understanding} (the final part of the proof there requires $E_r(x)$ to be symmetric, which is not the case here).
\squishend

\begin{corollary}[Error Confidence]\label{eq:corollary}
For the error
\begin{align*}
\epsilon = \mathrm{e} \, s^{-1}\cdot  Q_{n,p,\lceil\log(1/\delta)\rceil}(v),
\end{align*}
we have $\Pr[|E(x)|>\epsilon] \leqslant 1-\delta $ and \eqref{eq:jl_statement} holds.
\end{corollary}
The corollary is a direct application of Markov's inequality $\Pr[|E(x)|>\epsilon]\leqslant (s^{-1} Q_{n,p,d}(v) /\epsilon)^d$.

\subsection{Discussion}

\begin{remark}[Computational Efficiency]
The time of evaluating the distortion $\epsilon$ in \Cref{eq:corollary} is in
\begin{align}
\textsc{TIME} = O( \log^{4}(1/\delta)\log(m\log(1/\delta)) ).
\end{align}
\end{remark}

This is because $T_{n,p,d}(v)$ can be evaluated with $O(d^3)$ operations, utilizing the well-known combinatorial formulas for binomial moments~\cite{griffiths2013raw,knoblauch2008closed}. In turn, $Q_{n,p,d}(v)$ inverts a monotone function which can be computed by bisection in $O(\log (m\,d))$ steps (we have $Q \leqslant O(m\,d)$ which follows from $T_{n,p,d} \leqslant O(d)$). 

\begin{remark}[Comparison to State-of-the-Art]\label{rem:state_of_art}
The approach of~\cite{jagadeesan2019understanding} follows the same roadmap, but the critical steps in that work are estimated in a weaker way than in our approach, namely:
\begin{enumerate}
\item a weaker analogue of our \Cref{lemma:chaos} is used,
\item in place of our sharp \Cref{cor:extreme_point}, an overestimation of $\|\sum_i x_i Y_i\|_{d}$ is obtained,
\item bounds on $E_r(x)$ are assembled to bound $E(x)$ via a weaker variant of~\cite{latala1997estimation}, which is further weaker than our \Cref{lemma:iid_estimation}.
\end{enumerate}
\end{remark}

Thus, our bounds are guaranteed to be tighter for all parameter regimes.

\begin{remark}[Dependency on $n$]
Although our dependency on $n$ is only asymptotically bounded, we find that---interestingly---it indeed helps improving the bounds on real-data datasets and use-cases, as shown in the next section.
\end{remark}

A detailed empirical evaluation of our findings is provided in the following section.

\section{Empirical Evaluation}
\label{sec:experiments}

In this section, the present the detailed results of our experimental evaluation. We implemented the bound provided by \Cref{thm:main} in Python 3.6 and tested it in the Google Colab environment using an Intel(R) Xeon(R) CPU @ 2.20GHz and the default RAM configuration of 13GB.

\subsection{Implementation}
Computing $T_{n,p,d}(v)$ is non-trivial because of the moments of $B'-B''$ in \eqref{eq:row_estimate}, namely the naive expansion leads to an alternating sum. We stabilize it numerically by symbolic simplifications and by subtracting the leading term using \texttt{SymPy}. In turn, the implicit function $Q_{n,p,d}$ is evaluated by solving \eqref{eq:aggregated_func} via the bisection algorithm from \texttt{SciPy}. This yields a reasonably fast algorithm which on average takes about 1 millisecond per call in our Colab environment. 

This gives a reasonably fast algorithm (about 1 millisecond per call on average, implemented with \texttt{SymPy}/\texttt{SciPy} and run in Google Colab) as illustrated in \Cref{fig:running_time}. The plot shows the distribution of running times aggregated over different choices of all parameters $n, m, s, d, v$:
$n$ is sampled uniformly between $10^3$ and $10^6$,  $m$ is sampled uniformly between  $0.01 n$ and $n$,
$d$ is sampled uniformly between $2$ and $30$, finally $p$ and $v$ are sampled uniformly between $0$ and $1$.
\begin{figure}[ht!]
\centering
\includegraphics[width=0.45\linewidth]{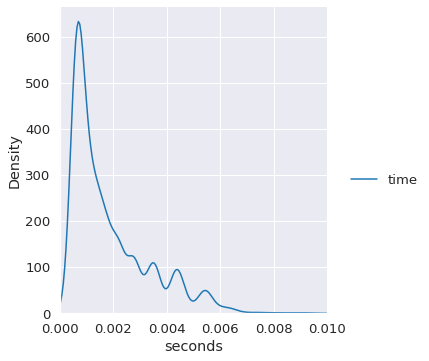}
\caption{Distribution of running times for the bounds of~\Cref{thm:main}}
\label{fig:running_time}
\end{figure}

\subsection{Baselines: Best Bounds in Prior Works}
To give a clear and fair comparison, we analyze the best constants in the previous asymptotic analysis~\cite{jagadeesan2019understanding}. The in-depth analysis gives the value of ``optimistic'' constants necessary to avoid breaking down the proof (while the actual constants are likely worse).

\begin{remark}[Optimistic Constants in Prior Works]\label{remark:best_constants}
The bound provided in~\cite{jagadeesan2019understanding} uses the better of the following two lemmas (Lemmas D.1 and D.2, respectively):
\begin{enumerate}
\item $\|E_r(x)\|_d \leqslant 2 \, C_1 \cdot  \left(\sup_{1\leqslant t\leqslant \frac{d}{2}} \left[\frac{d v}{t}\left(\frac{p}{dv^2}\right)^{\frac{1}{2t}}\right]\right)^2$
\item $\|E_r(x)\|_d \leqslant 2 \, C_2 \cdot \frac{d}{\log(1/p)}$,
\end{enumerate}
where $d$ is assumed positive and even.
\end{remark}

Here, the extra factor of $2$ appears as the effect of symmetrization (the random variable $E_r(x)$ must be dominated by a symmetric random variable to conclude the bound on $E(x)$). The best constants satisfy $C_1\geqslant 4\mathrm{e}$ and $C_2\geqslant 8$, as it is implied by the analysis of their proof technique.

\subsection{Synthetic Benchmark}

\noindent\textbf{Setup.~}
The key ingredient of our improvements is the sharper bound on the row-wise error contributions $E_{r}(x)$ from 
\Cref{lemma:row_estimate}. In this experiment, we compare this bound (referred to as $T_{new}$) with its analogue from~\cite{jagadeesan2019understanding} with the ``optimistic'' constants as discussed in \Cref{remark:best_constants}
(referred to as $T_{old}$). \Cref{fig:synthetic1,fig:synthetic2} illustrate the respective ratios of $T_{new}$ and  $T_{old}$ with respect to the error contributions $E_r(x)$ for $n=10^4$ and various ranges of $d$, $v$ and $p=\frac{s}{m}$. Points with non-even $d$ are interpolated. 

\medskip
\noindent\textbf{Results.~}
Our bounds are better by up to an order of magnitude across a wide range of parameters. Therefore, we should expect similar improvements for our bounds on the overall error $E(x)$ (recall that $E_r(x)$ are aggregated into $E(x)$ using 
\Cref{lemma:iid_estimation}).

\begin{centering}
\begin{minipage}{\textwidth}
\centering
\begin{minipage}[b]{0.49\textwidth}\centering
\includegraphics[width=0.95\linewidth]{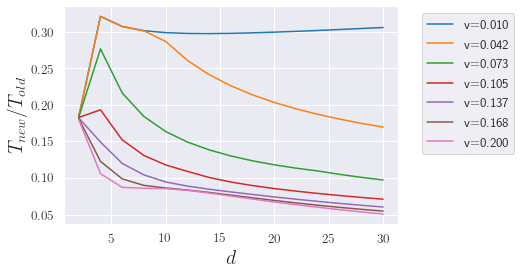}
\captionof{figure}{$T_{new}/T_{old}$ for $n=10^4$, $p=10^{-3}$}
\label{fig:synthetic1}
\end{minipage}
\hfill
\begin{minipage}[b]{0.49\textwidth}\centering
\includegraphics[width=0.96\linewidth]{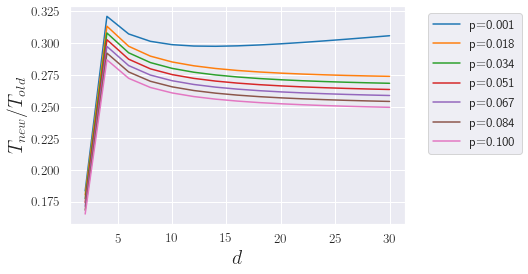}
\captionof{figure}{$T_{new}/T_{old}$ for $n=10^4$, $v=10^{-2}$}
\label{fig:synthetic2}
\end{minipage} 
\end{minipage}
\end{centering}

%

\subsection{Real-World Datasets}

\noindent\textbf{Setup.~}
We next consider various real-world datasets of different content types, sizes and numbers of features---as summarized in \Cref{tab:datasets}. Some of these datasets have small numbers of features, which is to demonstrate that our bounds give good results also when $n$ is small.

\begin{table}[htb!]
\vspace*{-1\baselineskip}
\centering
\resizebox{1.0\textwidth}{!}{
\begin{tabular}{|p{0.25\linewidth} | p{0.19\linewidth} | p{0.65\linewidth}|  }
\hline
\rowcolor{lightgray}
\textbf{Dataset} &\textbf{Content } & \textbf{Comments} \\
\hline
\texttt{NIPS}
 & text & 13,000 words \\
\hline
\texttt{Word2Vec/Wiki} & text &5M lines / 48M words of English Wikipedia articles processed by a default NLP pipeline of \texttt{Spacy} \cite{honnibal2018spacy}\\
\hline
\texttt{News20} & text & 20,000 documents / 34,000 words of English news \cite{Lang95}\\
\hline
\texttt{MNIST} & images & 60,000 images with 28x28 pixels \cite{lecun-mnisthandwrittendigit-2010} \\
\hline
\texttt{CIFAR100}& images & 60,000 images with 32x32 pixels \cite{krizhevsky2009learning}  \\
\hline
\texttt{SVHN }& images & 600,000 images with 32x32 pixels \cite{netzer2011reading} \\
\hline
\texttt{Caltech101} & images &9,000 images with 300x200 pixels \cite{fei2004learning}\\
\hline
\texttt{Cars} & images &16,000 images with 500x500 pixels \cite{krause20133d}\\
\hline
\texttt{Goodwin040} & fluid dynamics & 18,000 columns / 18,000 rows \cite{davis2011university}\\
\hline
\texttt{Mycielis\-kian17} & undir. graph & 98,000 columns / 98,000 rows \cite{davis2011university} \\
\hline
\end{tabular}
}
\caption{Summary of real-world datasets used in our experiments}
\label{tab:datasets}
\end{table}

\noindent\textbf{Dispersion.}
Since sparsity $s$ depends on the data-dependent dispersion $v$, results obtained in prior work may be of limited applicability in practice when $v$ is not small. To understand the behavior of $v$, we evaluate its distribution on our datasets. We conclude that, indeed, the value of $v$ may be quite large, even when $n$ is big; in such cases, using a very small sparsity $s$ is not theoretically justified.

Density plots on \Cref{fig:dispersion_text,fig:dispersion_small_images,fig:dispersion_large_images,fig:dispersion_misc} illustrate the distribution of the dispersion $v = \|x\|_{\infty}/\|x\|_2$ for vectors $x = x_1-x_2$ over all pairs $x_1,x_2 \in \mathcal{X}$ from a subsample $\mathcal{X}$ of the dataset. Evaluating the dispersion on pairwise differences corresponds to the intended usage of random projections:  preserving pairwise distances within a dataset. We used $|\mathcal{X}|=250$ so that $v$ is estimated based on $\approx 5 \cdot 10^4$ samples. 

We generally find that, for each dataset, $v$ is sharply concentrated around a ``typical'' value, whose magnitude is data-dependent. For example, data with smaller $n$ may be less dispersed than data with large $n$. Text, represented either by neural embeddings or bag-of-words, is usually quite dispersed.

\begin{centering}
\begin{minipage}{\textwidth}
\centering
\begin{minipage}[b]{0.49 \textwidth}\centering
\includegraphics[width=0.95\linewidth]{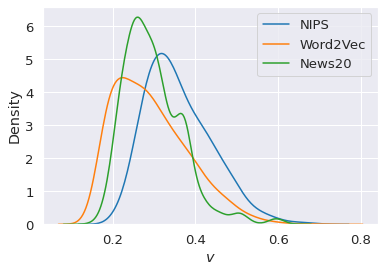}
\captionof{figure}{Dispersion $v$ on text data} 
\label{fig:dispersion_text}
\end{minipage}
\begin{minipage}[b]{0.49\textwidth}\centering
\includegraphics[width=0.98\linewidth]{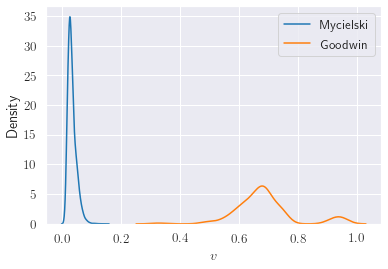}
\captionof{figure}{Dispersion $v$ on sparse-matrix} 
\label{fig:dispersion_misc}
\end{minipage}    
\end{minipage}
\begin{minipage}{\textwidth}
\centering
\begin{minipage}[b]{0.49\textwidth}\centering
\includegraphics[width=0.94\linewidth]{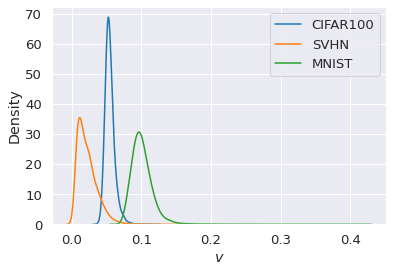}
\captionof{figure}{Dispersion $v$ on small images} 
\label{fig:dispersion_small_images}
\end{minipage}
\begin{minipage}[b]{0.49\textwidth}\centering
\includegraphics[width=0.98\linewidth]{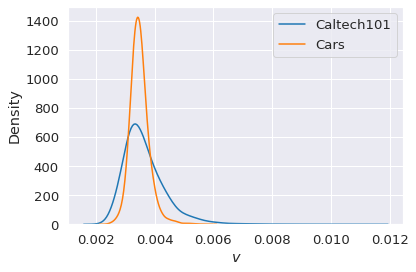}
\captionof{figure}{Dispersion $v$ on large images} 
\label{fig:dispersion_large_images}
\end{minipage}    
\end{minipage}
\end{centering}





\medskip
\noindent\textbf{Distortion.~}
The next experiment analyzes the confidence $1-\delta$ as a function of the distortion $\epsilon$ of our and previous bounds. We assume $\frac{m}{n} = 0.1$, $\frac{s}{m}=0.01$. The dispersion $v$ is chosen at the typical most likely value for each dataset 
(see our previous analysis). The confidence follows from \Cref{thm:main} by Markov's inequality. The results are illustrated on \Cref{fig:conf_eps_summary}. 

Our bounds produce very good results for all datasets with large $n$, thus outperforming the previous approach by several orders of magnitude in terms of confidence. Remarkably, we also obtain non-trivial bounds when $n$ is small (such as \texttt{SVHN}), as opposed to the bounds from previous works. For some datasets, prior bounds produce trivial results (i.e., $1-\delta=0$) for fairly large ranges of $\epsilon$.


\medskip
\noindent\textbf{Sparsity.~}
In this experiment, we evaluate the critical value of distortion $\epsilon$, which allows for using non-trivial sparsity $s<m$ such that the confidence $1-\delta$ is at least $\frac{3}{4}$. For each dataset, we choose as before its typical value $v$ and fix the dimension reduction factor $\frac{m}{n} = 0.1$. The results are summarized in \Cref{fig:eps_s}. Note that, for smaller values of $\epsilon$, no $s<m$ can work, which produces flat segments $s=m$ (particularly visible for previous bounds). Our bounds offer a non-trivial sparsity $s$ for much smaller distortions, and quickly achieve $s=1$.

\begin{centering}
\begin{minipage}{\textwidth}
\centering
\begin{minipage}[b]{0.49\textwidth}\centering
\includegraphics[width=0.95\linewidth]{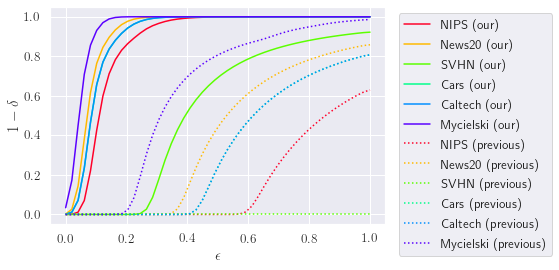}
\captionof{figure}{Confidence $1-\delta$ vs. distortion}
\label{fig:conf_eps_summary}
\end{minipage}
\hfill
\begin{minipage}[b]{0.49\textwidth}\centering
\includegraphics[width=0.99\linewidth]{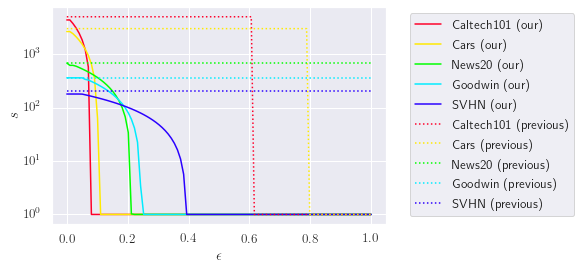}
\captionof{figure}{Sparsity $s$ vs. distortion}
\label{fig:eps_s}
\end{minipage}    
\end{minipage}
\end{centering}


\medskip
\noindent\textbf{Dimensionality.~}
In the last experiment, we evaluate the minimal non-trivial dimension $m$. We again consider a fixed sparsity of $\frac{s}{m} = 0.1$ and choose the typical dispersion $v$ for each dataset. Then, for various values of $\epsilon$, we compute the smallest $m$ which still yields a confidence of $1-\delta$ of $\frac{3}{4}$. The results, illustrated in \Cref{fig:dim_compar}, show that our bounds are better by 10 times or more. As the critical value of $m$ does not, at least asymptotically, depend on $n$ or $v$, we expect a similar behavior across datasets.

\medskip
\noindent\textbf{Multiple data points.~}
So far the experiments covered the performance on one input vector at a time only; the case of multiple data points reduces to the former one by scaling the confidence accordingly (union bound), where we again compute the smallest $m$ which still yields a confidence $1-\delta$ of $\frac{3}{4}$ over all points. The result shows the expected logarithmic dependency of the dimensionality $m$ with respect to the data size, as shown in \Cref{fig:dim_datasize}.

\begin{centering}
\begin{minipage}{\textwidth}
\centering
\begin{minipage}[b]{0.49\textwidth}\centering
\includegraphics[width=0.93\linewidth]{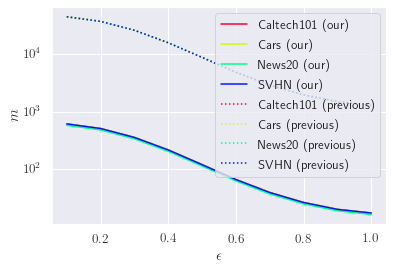}
\captionof{figure}{Dimensionality  vs. distortion}
\label{fig:dim_compar}
\end{minipage}
\hfill
\begin{minipage}[b]{0.49\textwidth}\centering
\includegraphics[width=0.99\linewidth]{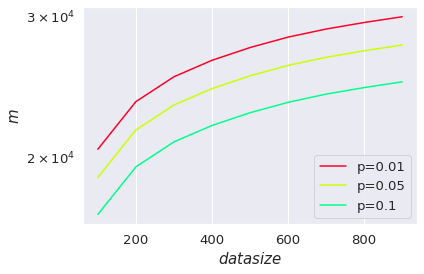}
\captionof{figure}{Dimensionality vs. data size}
\label{fig:dim_datasize}
\end{minipage}    
\end{minipage}
\end{centering}


\section{Conclusions}
We presented a framework for sparse random projections which provides provable guarantees with empirically significant numerical improvements over previous approaches. Our gain in comparison to previous approaches has been demonstrated on a large variety of (both synthetic and real-world) datasets. Moreover, we believe that the novel inequalities behind our improvements are of broader interest for a variety of statistical-inference applications.

\subsubsection*{Acknowledgements}
We thank the NVIDIA AI Technology Center (NVAITC) for the fruitful discussion.

\newpage

\appendix

\section{Proofs}

\subsection{Proof of \Cref{thm:main}}
\label{proof:thm:main}
By \eqref{eq:contrib} and \eqref{eq:row_contrib}, we have $E(x) =s^{-1} \sum_{r=1}^{m}E_r(x)$.
Let $E'_r(x) $ be independent copies of $E_r(x)$ for $r=1\ldots m$. 

The random variables $(E_r(x))_r$ are negatively dependent (see~\cite{dubhashi1996balls}), and thus the moments of their sum are not bigger than if they were independent. More precisely, for any positive integer $d$ we have:
\begin{align*}
\mathbb{E}(\sum_r E_r(x))^d \leqslant \mathbb{E}(\sum_r E'_r(x))^d
\end{align*}
as observed in~\cite{jagadeesan2019understanding} (see also a general argument in~\cite{shao2000comparison}). 
Thus, it holds that:
\begin{align*}
\|\sum_r E_r(x)\|_d \leqslant \|\sum_r E'_r(x)\|_d.
\end{align*}
The random variables $(E'_r(x))_r$ are iid with moments bounded by the moments of some symmetric random variables $E''_r(x)$ as discussed in \Cref{rem:chaos}. The moments of $E''_r(x)$ are in turn estimated in~\Cref{lemma:row_estimate}. The result follows now by applying \Cref{lemma:iid_estimation}.

\subsection{Proof of \Cref{lemma:chaos}}
\label{proof:lemma:chaos}
By the decoupling inequality, for any integer $d>0$, we obtain:
\begin{align}
   \|\sum_{i\not=j} Z_i Z_j\|_d \leqslant 4\|\sum_{i\not=j} Z_i Z'_j\|_d
\end{align}
where $Z'_i$ are independent copies of $Z_i$.
Next, we claim that for even $d$ the following holds:
\begin{align}
 \mathbb{E}(\sum_{i\not=j} Z_i Z'_j)^d \leqslant  \mathbb{E}(\sum_{i,j} Z_i Z'_j)^d
\end{align}
This follows by the multinomial expansion applied to both sides and evaluating the expectation term-by-term; due to the symmetry of random variables $Y_i$ and $Y'_i$, the expectation of every term is either zero or non-negative. Since $Y'_i$ and $Y_i$ are identically distributed, the sum on the right-hand side contains all the terms that appear on the left-hand side.

\subsection{Proof of \Cref{lemma:chaos_schur}}
\label{proof:lemma:chaos_schur}
Here, it suffices to prove that
$$
u\to  \mathbb{E}f\left(\sum_i Y_i u_i^{1/2}\right)
$$
is Schur-concave in $u$, where $f(t)=t^d$. Indeed, we have $\mathbb{E}(\sum_i Y_i x_i)^d = \mathbb{E}(\sum_i Y_i |x_i|)^d$ (follows by raising to the power of $d$ and applying the multinomial expansion, then only even powers contribute to the expectation), and
the claim follows by denoting $|x_i| = u_i$.

Since $g$ is symmetric, it suffices to check the Schur-Ostrowski criterion~\cite{ostrowski1952quelques,schur1923uber} for $u_1$ and $u_2$.
Let $W = \sum_{i>2} \eta_i\sigma_i u_i^{1/2}$, then
\begin{align*}
\frac{\partial g}{\partial u_1}-\frac{\partial g}{\partial u_2} = \frac{u_2^{1/2} X_1 - u_1^{1/2}X_2}{2(u_1u_2)^{1/2}}\cdot  f'\left(\sum_{i=1}^{2} Y_i  u_i^{1/2}+W\right)
\end{align*}
Thus it remains to prove that the expectation of
\begin{align}\label{eq:schur_proof_1}
Q\triangleq (u_2^{1/2} X_1 - u_1^{1/2}X_2)\cdot  f'\left(\sum_{i=1}^{2} Y_i  u_i^{1/2}+W\right)
\end{align}
is negative when $u_1 < u_2$.
Recall that $Y_i$ are symmetric and take three values $\{-1,0,1\}$.
We condition on two cases: a) $X_1,X_2\not=0$ and b) one of $X_1,X_2$ is zero.
In case a) the result reduces to the case of Rademacher variables, solved already by Eaton~\cite{eaton1970note}.
We are left with case b). If $X_1=X_2$ the expression is zero. We further assume $X_1\not=X_2$.
Consider the two disjoint events: $\mathcal{E}_1$ is that $X_2=0$ and $X_1=\pm 1$ and $\mathcal{E}_2$ is that $X_1=0$ and $X_2=\pm 1$. 
Then we have that
\begin{align*}
\mathbb{E}\left[ Q | \mathcal{E}_1,W \right] &= u_2^{1/2}\left( f'\left(  u_1^{1/2}+W\right)-f'\left(  -u_1^{1/2}+W\right)  \right)\\
\mathbb{E}\left[ Q | \mathcal{E}_2,W \right] &= -u_1^{1/2}\left( f'\left(  u_2^{1/2}+W\right)-f'\left(  -u_2^{1/2}+W\right)  \right)
\end{align*}
For $t>0$ we consider the the auxiliary function
\begin{align*}
g_w(t) = t^{-1}\left( f'\left(  t+w\right)-f'\left(  -t+w\right)  \right)
\end{align*}
We have $\mathbb{E}\left[ Q | \mathcal{E}_1,W \right] = (u_1 u_2)^{1/2} g_W(u_1^{1/2})$
and $\mathbb{E}\left[ Q | \mathcal{E}_2,W \right] = -(u_1 u_2)^{1/2} g_W(u_2^{1/2})$, and therefore
$\mathbb{E}[Q|\mathcal{E}_1\cup\mathcal{E}_2] = (u_1 u_2)\mathbb{E}_W[g_W(u_1^{1/2})-g_W(u_2^{1/2})]$.
If we prove that $g_W(t)$ increases in $t$, the proof is complete.

Since in our case $f(t) = t^{d}$, we find that $g_w(t) = d \cdot \frac{(w+t)^{d-1}-(w-u)^{d-1}}{t}$.
Since $d$ is even $g_w(t) = d t^{-1}( (w+t)^{d-1} + (t-w)^{d-1}) = d\cdot \sum_{0\leqslant k<\frac{d-1}{2}} \binom{d}{2k} t^{d-2-2k} w^{2k}$, indeed is increasing in $t$ regardless of $w$.

\subsection{Intuitions about \Cref{cor:extreme_point}}

The result follows because $x^{*}$ is majorized by every other vector which satisfies the constraints. However, one may wonder why the flat vector $x^{flat}$, with all non-zero entries equal to $v$, is not the worst case?
Observe that already for the case of $d=2$ this gives the norm of $\sqrt{v^2 p}$ while our construction gives the bigger value $\sqrt{p v^2 + p\frac{1-v^2}{n-1} }$.

\subsection{Proof of \Cref{lemma:row_estimate}}
\label{proof:lemma:row_estimate}
Due to \Cref{lemma:chaos} applied to $Z_i \sim A_{r,i}\cdot x_i$ and the definition of $A$, it suffices to show that $\|\sum_i x_i Y_i \|_d \leqslant T_{n,p,d}(v)$.
Consider $x^{*}$ as in \Cref{cor:extreme_point}.
Define $W = \frac{1-v^2}{\sqrt{n-1}}\sum_{i=2}^{n} Y_i$.
Using the independence and symmetry of $Y_1$ and $W$, we obtain:
\begin{align*}
\|\sum_i x^{*}_i Y_i \|_d^d &\leqslant \mathbb{E}\left(Y_1 v + W \right)^d \\
& = \sum_{k=0}^{d/2}\binom{d}{2k} v^{2k}\mathbb{E}Y_1^{2k} \mathbb{E}W^{d-2k} 
\end{align*}
Let $B_1,B_2$ be Bernoulli with parameter $\sigma$ such that
$\sigma^2+(1-\sigma)^2 = 1-p$, then $Y_i \sim B_1-B_2$, and thus
$W\sim B'-B''$ where $B,B''\sim^{i.i.d.} \mathrm{Binom}(n-1,\sigma)$. Also, we have $\mathbb{E}Y_1^{2k} = p^{\mathbb{I}(k>0)}$.
We have $T_{n,p,d}(v)\leqslant \|\sum_i x^{*}_i Y_i \|_d$ which (combined with the bound above) completes the proof.

\subsection{Analysis of~\cite{jagadeesan2019understanding}}

\paragraph{Proof of \Cref{remark:best_constants}}
\label{proof:remark:best_constants}
Inspecting the proof, we find that the best constants are:
\begin{enumerate}
\item $C_1=2\mathrm{e}C$ where $C$ is a constant which satisfies
$\binom{2d}{2d_1\ldots 2d_n} \leqslant C^{d} \frac{(2d)^{2d}}{\prod_i (2d_i)^{2d_i}}$ for all $(d_i)_i$ with sum $d$ and such that 
$d_i\leqslant d/2$. Specializing to $d_1 = d/2, d_2 = d/2$ and $d_i=0$ for $i>2$ we 
see that $C$ must satisfy $\binom{2d}{d} \leqslant C^d \frac{(2d)^{2d}}{d^{2d}} = C^d 2^{2d}$. But $\binom{2d}{d}=\Theta(2^{2d}/\sqrt{d})$ (see~\cite{eger2014stirling}), so we must have $C\geqslant 1$.

\item The proof starts with the bound $\|E_r(x)\|_{d}^{1/2} \leqslant K^{1/2} + C_1 \cdot K^{-1/2}\sup_{1\leqslant t\leqslant d} d/t\cdot (K p /d )^{1/2t} $ for any integer $K$ and $C_1\geqslant 2\mathrm{e}$ as in the discussion above. The goal is to choose $K$ so that the right-hand side becomes $\sqrt{C _2 d/\log(1/p)}$. Following the derivative test, the value of $t$ is optimized
by substituting $t=\frac{1}{2}\log(d/Kp)$ which gives the value of $R=K^{1/2}+\frac{2d}{K^{1/2}\log(d/Kp)}$.
Now, by the inequality of arithmetic and geometric means, we obtain $R \geqslant 2\sqrt{2d/\log(d/Kp)} \geqslant 2\sqrt{2d/\log(1/p)}$. Thus, the proof implies only $C_2 \geqslant 8$.

\end{enumerate}



\subsection{Proof of \Cref{lemma:iid_estimation}}
\label{proof:lemma:iid_estimation}
Consider random variables $Z_1,\ldots,Z_n$. We have
\begin{align*}
\mathbb{E}(\sum_i Z_i)^d = \sum_{d=(d_i)_i}\binom{d}{d_1\ldots d_n}\prod_{i}\mathbb{E}Z_i^{d_i}
\end{align*}
We need the following
\begin{proposition}\label{prop:better_latala}
We have $\binom{d-y}{x} \leqslant c\cdot \binom{d}{x}$ for $0\leqslant x,y$, $x+y\leqslant d$ where
$c= \mathrm{e}^{-\frac{xy}{d}}$.
\end{proposition}
\begin{proof}
Note that $c$ satisfies $\prod_{i=0}^{x-1} \left(1-\frac{y+i}{d+i}\right) \leqslant c$, where the left-hand side is at most
$(1-y/d)^{x} \leqslant \mathrm{e}^{- y x/d}$.
\end{proof}
We conclude that
\begin{align*}
\binom{d}{d_1\ldots d_n} \leqslant \mathrm{e}^{-\frac{d}{2} } \prod_{i=1}^{k}\binom{d}{d_i}.
\end{align*}
To see this, we assume without losing generality that $d_i$ is sorted in the descending order. Since
$
\binom{d}{d_1\ldots d_n} = \binom{d}{d_1}\binom{d-d_1}{d_2}\binom{d-d_1-d_2}{d_3}\ldots
$ by \Cref{prop:better_latala} the above holds with constant $\mathrm{e}^{-c}$ where $c=d^{-1}\sum_{1\leqslant j\leqslant i\leqslant d}^{k} d_i d_j \geqslant \frac{(\sum_{i} d_i)^2}{2d} =d/2 $. 

Using the above bound, we get
\begin{align*}
\mathbb{E}(\sum_i Z_i)^d &\leqslant \mathrm{e}^{-d/2} \sum_{d=(d_i)_i}\prod_i\binom{d}{d_i}\mathbb{E}Z_i^{d_i} \\
& =\mathrm{e}^{-d/2} \prod_i \sum_k \binom{d}{k}\mathbb{E}Z_i^k,
\end{align*}
where the symmetry assumption is used to ensure that $\mathbb{E}Z_i^{d_i}\geqslant 0$. Substituting $Z_i:= Z_i/t$, we obtain
\begin{align*}
\mathbb{E}(t^{-1}\sum_i Z_i)^d &\leqslant \mathrm{e}^{-d/2} \prod_i \sum_k \binom{d}{k}\mathbb{E}Z_i^k/t^k
\end{align*}
Now, if $Z_1,\ldots, Z_n\sim^{i.i.d.} Z$ and $t$ is such that
$(\sum_k \binom{d}{k}\mathbb{E}Z^k / t^k)^n = \mathrm{e}^{d/2}$, we obtain $\mathbb{E}(t^{-1}\sum_i Z_i)^d\leqslant 1$ which is equivalent to $\|\sum_i Z_i\|_d \leqslant t$.

\section{Other Results}
\subsection{Latala's Framework for I.I.D. Sums}
\begin{lemma}[cf. Corollary 2 in ~\cite{latala1997estimation}]
Let $X_1,\ldots,X_n$ be symmetric RVs with common distribution $X$. Then, the following holds:
\begin{multline}
\|X_1+\ldots + X_n\|_d = \\ \Theta(1)\cdot\sup\left\{ \frac{d}{t}\cdot \left(\frac{n}{d}\right)^{1/t}\cdot \|X\|_t:\ \max\{2,\frac{d}{n}\}\leqslant t\leqslant d \right\}
\end{multline} 
\end{lemma}

\newpage
\bibliographystyle{plain}
\bibliography{citations}

\end{document}